\documentclass{article}

\usepackage{microtype}
\usepackage{graphicx}
\usepackage{subfigure}
\usepackage{booktabs} 

\usepackage{hyperref}

\usepackage{amsthm,amsmath,amsfonts}
\newcommand{\R}{\mathbb{R}}

\newcommand{\E}{\mathbb{E}}

\newcommand{\X}{\mathcal{X}}

\newcommand{\A}{\mathcal{A}}

\newcommand{\Prob}{\mathbb{P}}
\providecommand{\abs}[1]{\lvert#1\rvert}

\usepackage{xcolor}

\usepackage[accepted]{icml2021}

\icmltitlerunning{Refined Policy Improvement Bounds for MDPs}

\newtheorem{lemma}{Lemma}
\newtheorem{theorem}{Theorem}
\newtheorem{corollary}{Corollary}

\theoremstyle{definition}

\begin{document}

\twocolumn[
\icmltitle{Refined Policy Improvement Bounds for MDPs}




\begin{icmlauthorlist}
\icmlauthor{J. G. Dai}{sribd,orie}
\icmlauthor{Mark Gluzman}{cam}
\end{icmlauthorlist}

\icmlaffiliation{sribd}{School of Data Science,  Shenzhen Research Institute of Big Data, The Chinese
University of Hong Kong, Shenzhen, China}
\icmlaffiliation{orie}{School of Operations Research and Information Engineering, Cornell University, Ithaca, USA}
\icmlaffiliation{cam}{Center for Applied Mathematics, Cornell University, Ithaca, USA}

\icmlcorrespondingauthor{Mark Gluzman}{mg2289@cornell.edu}

\icmlkeywords{average reward reinforcement learning, policy improvement bound}

\vskip 0.3in
]



\printAffiliationsAndNotice{}  

\begin{abstract}
  The policy improvement bound on the difference of the discounted
  returns plays a crucial role in the theoretical justification
  of the trust-region policy optimization (TRPO) algorithm. The
  existing bound leads to a degenerate bound when the discount
    factor approaches one, making the applicability of TRPO
  and related algorithms questionable when the discount factor is
  close to one. We refine the results in \cite{Schulman2015,
    Achiam2017} and propose a novel bound that is ``continuous'' in
  the discount factor. In particular, our bound is applicable
  for MDPs with the long-run average rewards as well.
\end{abstract}

 \section{Introduction}



 In \cite{Kakade2002} the authors developed a conservative
   policy iteration algorithm for Markov decision processes (MDPs)
   that can avoid catastrophic large policy updates; each iteration
   generates a new policy as a mixture of the old policy and a greedy
   policy. They proved that the updated policy is guaranteed to
   improve when the greedy policy is properly chosen and the updated
   policy is sufficiently close to the old one.  In
   \cite{Schulman2015} the authors generalized the proof of
   \cite{Kakade2002} to a policy improvement  bound for two
   \emph{arbitrary} randomized policies.  This policy improvement
   bound allows one to find an updated policy that guarantees to
   improve by solving an unconstrained optimization problem.
   \cite{Schulman2015} also proposed a practical algorithm, called
   trust region policy optimization (TRPO), that approximates the
   theoretically-justified update scheme by solving a constrained
   optimization problem in each iteration.  In recent years, several
 modifications of TRPO have been proposed \cite{Schulman2016,
   Schulman2017, Achiam2017, Abdolmaleki2018}. These studies continued
 to exploit the policy improvement bound to theoretically motivate
 their algorithms.

  The policy improvement bounds in \cite{Schulman2015,
     Achiam2017} are   lower bounds on the difference of the expected
   \emph{discounted returns} under two policies.  Unfortunately, the
 use of these policy improvement bounds becomes questionable and
 inconclusive when the discount factor is close to one. These policy
 improvement bounds degenerate as discount factor converges to one.
 That is, the lower bounds on the difference of discounted returns
 converge to negative infinity as the discount factor goes to one,
 although the difference of discounted returns converges to the
 difference of (finite) average rewards.  Nevertheless, numerical experiments demonstrate that the TRPO
 algorithm and its variations perform best when the discount factor
 $\gamma$ is close to one, a region that the existing
 bounds do not justify;  e.g. \cite{Schulman2015, Schulman2016,
   Schulman2017} used $\gamma=0.99$, and \cite{Schulman2016,
   Achiam2017} used $\gamma=0.995$ in their experiments.

 Recent studies \cite{Dai2020, Zhang2021}
 proposed policy improvement bounds for average rewards, showing that
 a family of TRPO algorithms can be used for continuing problems with
 long-run average reward objectives.  Still it remains unclear how the
 large values of the discount factor can be justified and why the
 policy improvement bounds in \cite{Schulman2015, Achiam2017} for the
 discounted rewards do not converge to one of the bounds provided in
 \cite{Dai2020, Zhang2021}.

  In this study, we provide a unified derivation of policy
   improvement bounds for both discounted and average reward MDPs. Our
   bounds depend on the discount factor \emph{continuously}.  When the
   discount factor converges to $1$, the corresponding
   bound for discounted returns converges to a policy improvement
   bound for average rewards.  We achieve these results by two innovative
   observations. First, we embed the discounted future state
   distribution under a fixed policy as the stationary distribution of
   a  modified Markov chain.  Second, we introduce an
   \emph{ergodicity coefficient} from Markov chain perturbation theory
   to bound the one-norm of the difference of discounted future state
   distributions, and prove that this bound is optimal in a certain
   sense. 
   Our results justify the use of a large discount
   factor in TRPO algorithm and its variations.

\section{Preliminaries}

We consider an MDP defined by the tuple $(\X, \A, P, r, \mu)$, where $\X$ is a  finite state space; $\A$ is a finite action space; $P$ is the transition probability function,
$r:\X\times\A\rightarrow \R$ is the reward function; $\mu$ is the probability distribution of the initial state $x_0$.

 We let $\pi$ denote a stationary randomized policy $\pi:\X\rightarrow \Delta(\A)$, where $\Delta(\A)$ is the probability simplex over $\A$.  Under policy $\pi$, the corresponding Markov chain has a transition matrix $P^\pi$ given by
$ P^\pi(x,y):=\sum\limits_{a\in\A }\pi(a|x)P(y|x,a),~x,y\in \X.
$
 We assume that MDPs we consider are unichain, meaning that for any stationary policy $\pi$ the corresponding  Markov chain with transition matrix $P^\pi$ contains only one recurrent class \cite{Puterman2005}.
 We use $d^\pi$ to denote a unique stationary distribution of  Markov chain with transition matrix $P^\pi$.


For a vector $a$ and a matrix $A$,
   $a^T$ and $A^T$ denote their transposes. For a vector $a$, we use the following vector norm: $\|a\|_1 :=\sum\limits_{x\in \X}\abs{a(x)}$. For a matrix $A$, we define the following induced norm: $\|A\|_1 := \max\limits_{y\in \X}\sum\limits_{x\in \X}\abs{A(x,y)}$. 


\subsection{MDPs with infinite horizon discounted returns}

We let  $\gamma\in [0,1)$ be a discount factor.  We define the value function for a given policy $\pi$ as
  \begin{align*}
V_\gamma^\pi(x) :=\E    \left[ \sum\limits_{t=0}^\infty \gamma^tr(x_t, a_t) \Big| \pi, x_0=x \right],
 \end{align*}
 where $x_t$, $a_t$ are random variables for the state and action at time $t$ upon executing the policy $\pi$ from the initial state $x$. For policy $\pi$ we define the state-action value function as  $Q_\gamma^\pi(x, a) :=r(x, a) + \gamma \E_{y\sim P^\pi(\cdot|x,a)} \left[V^\pi_\gamma(y) \right],$
 and  the advantage function as
$
A_\gamma^\pi(x, a):  =Q_\gamma^\pi(x, a) -  V_\gamma^\pi(x).
$



We define  the discounted future state distribution of policy $\pi$ as

$
d_\gamma^\pi(x):=(1-\gamma)\sum\limits_{t=0}^\infty \gamma^t\Prob\Big[x_t=x|x_0\sim \mu; x_1, x_2,...\sim \pi\Big].
$

 We measure the performance of policy $\pi$ by its expected discounted   return from the initial state distribution $\mu$:
  \begin{align*}
 \eta_\gamma^\pi (\mu):=(1-\gamma)\E_{x\sim \mu} [V_\gamma^\pi(x)]= \E_{ x\sim d_\gamma^\pi,  a\sim \pi(\cdot|x)} [ r(x, a) ].
 \end{align*}

 In the following lemma we give an alternative definition of  the discounted future state distribution as a stationary distribution of a modified transition matrix. 

  \begin{lemma}\label{lem:stat}
 For a stationary policy $\pi$, we define a discounted transition matrix for policy $\pi$ as
 \begin{align}\label{eq:mod_trans}
 P^\pi_\gamma:=\gamma P^\pi + (1-\gamma) e\mu^T,
 \end{align}
 where $e := (1,1,..,1)^T$ is a vector of ones, $e\mu^T$ is the matrix which rows are equal to $\mu^T$.

 Then the discounted future state distribution of policy $\pi$, $d_\gamma^\pi$, is the stationary distribution of transition matrix $P^\pi_\gamma$.

 \end{lemma}
  \begin{proof}[\textbf{Proof of Lemma \ref{lem:stat}}]
 We need to show that $(d_\gamma^\pi)^TP^\pi_\gamma=(d_\gamma^\pi)^T$. Indeed, we get
 \begin{align*}
(d_\gamma^\pi)^TP^\pi_\gamma&= (1-\gamma)\mu^T\sum\limits_{t=0}^\infty (\gamma P^\pi)^tP^\pi_\gamma\\
&=(1-\gamma)\mu^T\sum\limits_{t=0}^\infty (\gamma P^\pi)^t\Big(  \gamma P^\pi  + (1 - \gamma) e\mu^T \Big)\\
&=(1-\gamma)\mu^T\sum\limits_{t=0}^\infty (\gamma P^\pi)^{t+1}+(1-\gamma)^2\mu^T \sum\limits_{t=0}^\infty \gamma^n e\mu^T\\
& = (1-\gamma)\mu^T\sum\limits_{t=0}^\infty (\gamma P^\pi)^{t+1}+(1-\gamma)\mu^T\\
&=(1-\gamma)\mu^T \Big(\sum\limits_{t=0}^\infty (\gamma P^\pi)^{t+1}+I \Big)\\
&=(1-\gamma)\mu^T \sum\limits_{t=0}^\infty (\gamma P^\pi)^{t}\\
&= (d_\gamma^\pi)^T.
\end{align*}
 \end{proof}




  \subsection{MDPs with long-run average rewards}

 The long-run average reward of policy $\pi$ is defined as
\begin{align*}
 \eta^\pi :&= \lim\limits_{N\rightarrow \infty}\frac{1}{N} \E \left[ \sum\limits_{t=0}^{N-1} r(x_t,a_t) ~|~\pi, x_0\sim \mu\right] \\
&=  \E_{ x\sim d^\pi,  a\sim \pi(\cdot|x)}   \left[ r(x,a) \right].
 \end{align*}


For an MDP with a long-run average reward objective we define  the relative value function
 \begin{align*}
 V^\pi(x) :=\lim\limits_{N\rightarrow \infty}\E   \left[ \sum\limits_{t=0}^{N-1} (r(x_t, a_t) - \eta^\pi)~|~\pi, x_0=x \right],
 \end{align*}
the relative state-action value function $
 Q^\pi(x, a) :=r(x, a) - \eta^\pi + \E_{y\sim P^\pi(\cdot|x,a)}\left[V^\pi(y)\right],
$ and  the  relative advantage function $
 A^\pi(x, a): =    Q^\pi(x, a) -V^\pi(x)$. The following relations hold for value, state-action value, and advantage functions.

\begin{lemma}\label{lem:adv}

We let $\pi$ be a stationary policy, $\gamma$ be the discount factor, and $\mu$ be the initial state distribution. Then the following limits hold for each $x\in \X$, $a\in \A$:

$\eta^\pi = \lim\limits_{\gamma\rightarrow 1} \eta^\pi_\gamma(\mu) $, $V^\pi(x) =\lim\limits_{\gamma \rightarrow 1} \left( V^\pi_\gamma(x)  - (1-\gamma)^{-1} \eta^\pi\right)$, $ Q^\pi(x,a)=\lim\limits_{\gamma \rightarrow 1} \left( Q^\pi_\gamma(x, a)  - (1-\gamma)^{-1}\eta^\pi\right)$, and $A^\pi(x,a)=\lim\limits_{\gamma \rightarrow 1}A^\pi_\gamma(x,a)$.

\end{lemma}
The proofs of identities for the average rewards and value functions can be found in Section 8 in \cite{Puterman2005}. The rest results  follow directly.  

\section{ Novel Policy Improvement Bounds}


The policy improvement bound in \cite{Schulman2015, Achiam2017} for  the discounted  returns serves to theoretically justify the TRPO algorithm and its variations.
The following lemma is a reproduction of Corollary 1 in \cite{Achiam2017}.
\begin{lemma}\label{lem:perf_bound}
For any  two policies $\pi$ and $\tilde \pi$ the following bound holds:
\begin{align}\label{eq:perf_bound}
\eta_\gamma^{\tilde \pi}(\mu) -  &\eta_\gamma^\pi(\mu)\geq \underset{    \substack{ x\sim d_\gamma^{  \pi}, a \sim \tilde \pi(\cdot|x) }     }{\E} \left[A_\gamma^\pi(x,a)   \right] \\
&    -\frac{ 2\gamma \epsilon_\gamma^{\tilde \pi}}{1-\gamma}   \underset{  x\sim d_\gamma^{ \pi} }{\E} \left[\text{TV}\Big(\tilde \pi(\cdot|x)~||~\pi(\cdot|x)\Big) \right],\nonumber
\end{align}
where ${\text{TV}}\Big(\tilde \pi(\cdot|x)||\pi(\cdot|x)\Big):=\frac{1}{2}\sum\limits_{a\in \A} \abs{\tilde \pi(a|x) - \pi(a|x)}$, and $\epsilon_\gamma^{\tilde \pi}:=\max\limits_{x\in \X}\Big|\underset{a\sim \tilde \pi(\cdot|x)}{\E}[A_\gamma^\pi(x,a)]\Big|$.

\end{lemma}

The left-hand side of (\ref{eq:perf_bound}) converges to the difference of average rewards as $\gamma\rightarrow 1$. Unfortunately, the right-hand side of (\ref{eq:perf_bound})  converges to the negative infinity   because of $(1-\gamma)^{-1}$ factor in the second term.
Our goal is to get a new policy improvement bound for discounted returns that does not degenerate.


The group inverse $D$ of a matrix $A$ is the unique matrix such that
$
ADA = A,~ DAD=D,  \text{ and } DA=AD.
$ From \cite{Meyer1975}, we know that if stochastic matrix $P$ is aperiodic and  irreducible  then the group inverse matrix of $I-P$ is well-defined and equals to
$
 D = \sum\limits_{t=0}^\infty (P^t - ed^T),
$ where $d$ is the stationary distribution of $P$.

We let $D_\gamma^\pi$ be the group inverse of matrix $I-P_\gamma^\pi$, where $P_\gamma^\pi$ is defined by (\ref{eq:mod_trans}).
Following \cite{Seneta1991}, we define a one-norm ergodicity coefficient for a matrix $A$ as
\begin{align}\label{eq:erg_def}
\tau_1[A]:= \underset{ \substack{ \|x\|_1=1\\ x^Te = 0}}{\max}\|A^Tx\|_1.
\end{align}
 The one-norm ergodicity coefficient has two important properties. First, 
\begin{align}\label{eq:pr1}
\|A^Tx\|_1\leq \tau_1[A]\|x\|_1,
\end{align} for any matrix $A$ and vector $x$ such that $x^Te=0.$
Second, $\tau_1[A] = \tau_1[A+ec^T]$, for any vector $c$.
By Lemma \ref{lem:matr_diff} below, $\tau_1\left[D_\gamma^\pi\right] = \tau_1\left[(I-\gamma P^\pi)^{-1}\right]$, for $\gamma<1$.

\begin{lemma}\label{lem:matr_diff}
We let $\pi$ be an arbitrary policy. Then
\begin{align*}
D_\gamma^{ \pi} = (I-\gamma P^\pi)^{-1}+e (d_\gamma^\pi)^T( I - \left(I - \gamma P^\pi \right)^{-1}) - e(d^\pi)^T.
\end{align*}
\end{lemma}

We are ready to state the main result of our study.

  \begin{theorem}\label{thm:main}
The following bound on the difference of discounted returns of two policies $\pi$ and $\tilde \pi$  holds:
\begin{align}\label{eq:bound_opt}
\eta_\gamma^{\tilde \pi}(\mu) -  &\eta_\gamma^\pi(\mu)\geq \underset{    \substack{ x\sim d_\gamma^{  \pi}, a \sim \tilde \pi(\cdot|x) }     }{\E} \left[A_\gamma^\pi(x,a)   \right] \\
&    - 2\gamma \epsilon_\gamma^{\tilde \pi} \tau_1\left[D_\gamma^{\tilde \pi}\right] \underset{  x\sim d_\gamma^{ \pi} }{\E} \left[\text{TV}\Big(\tilde \pi(\cdot|x)~||~\pi(\cdot|x)\Big) \right].\nonumber
\end{align}
\end{theorem}
We provide a sketch of the proof of Theorem \ref{thm:main}.

  \begin{proof}[\textbf{Proof of Theorem \ref{thm:main}}]
We   closely follow the first steps in the proof of Lemma 2 in \cite{Achiam2017} and start with
\begin{align*}
\eta_\gamma^{\tilde \pi}(\mu) -  &\eta_\gamma^\pi(\mu)\geq \underset{    \substack{ x\sim d_\gamma^{  \pi}, a \sim \tilde \pi(\cdot|x) }     }{\E} \left[A_\gamma^\pi(x,a)   \right] \\
&    -   \max\limits_{x\in \X}\Big|\underset{a\sim \tilde \pi(\cdot|x)}{\E}[A_\gamma^\pi(x,a)]\Big| \|d_\gamma^\pi - d_\gamma^{\tilde \pi}\|_1.\nonumber
\end{align*}
Next,  unlike \cite{Achiam2017}, we obtain an upper bound on $ \|d_\gamma^{\tilde  \pi}-d_\gamma^{ \pi}\|_1$ that does not degenerate as $\gamma\rightarrow 1$. We use the following  perturbation identity:
\begin{align}\label{eq:old_new}
(d_\gamma^{\tilde  \pi})^T-(d_\gamma^{ \pi})^T =\gamma (d_\gamma^{   \pi})^T (P^{  \pi} - P^{ \tilde  \pi})D_\gamma^{\tilde \pi}.
\end{align}
Identity (\ref{eq:old_new}) follows from the perturbation identity for stationary distributions, see equation (4.1) in \cite{Meyer1980}, and the fact that $d_\gamma^{\tilde  \pi}$ and $d_\gamma^{ \pi}$ are the stationary distributions of the discounted transition matrices $P^{\tilde \pi}_\gamma$ and $P^\pi_\gamma$, respectively.
We make use of the ergodicity coefficient (\ref{eq:erg_def}) to get a new perturbation bound:
\begin{align}
\|d_\gamma^{\tilde  \pi}&-d_\gamma^{ \pi} \|_1 =\gamma \left\| \left(D_\gamma^{\tilde \pi}\right)^T(P^{  \pi} - P^{ \tilde  \pi})^Td_\gamma^{   \pi} \right\|_1 \nonumber \\
&\leq \gamma \tau_1 \left[D_\gamma^{\tilde \pi}  \right]\left\|(P^{  \pi} - P^{ \tilde  \pi})^Td_\gamma^{   \pi} \right\|_1 \label{eq:inv}\\
& \leq 2\gamma\tau_1\left[D_\gamma^{\tilde \pi}\right] \underset{  x\sim d_\gamma^{ \pi} }{\E} \left[\text{TV}\Big(\tilde \pi(\cdot|x)~||~\pi(\cdot|x)\Big) \right],\nonumber
\end{align}
where inequality (\ref{eq:inv}) holds due to (\ref{eq:pr1}) and equality $(P^{  \pi} - P^{ \tilde  \pi})e=0.$

 \end{proof}

The novel policy improvement bound (\ref{eq:bound_opt}) converges to a meaningful bound on the difference of average rewards as $\gamma$ goes to 1. Corollary 1 follows from Theorem \ref{thm:main}, Lemma \ref{lem:adv} and the fact that $\tau_1\left[D_\gamma^{\tilde \pi}\right] \rightarrow \tau_1\left[D^{\tilde \pi}\right]$ as $\gamma\rightarrow 1$.

  \begin{corollary}\label{thm:main_av}
The following bound on the difference of long-run average rewards of two policies $\pi$ and $\tilde \pi$  holds:
\begin{align}\label{eq:bound_av}
 \eta^{\tilde \pi} -  &\eta^\pi  \geq  \underset{    \substack{ x\sim d^{  \pi}, a \sim \tilde \pi(\cdot|x) }     }{\E} \left[A^\pi(x,a)   \right]\\
&  - 2 \epsilon^{\tilde \pi} \tau_1\left[D^{\tilde \pi}\right] \underset{  x\sim d^{ \pi} }{\E} \left[\text{TV}\Big(\tilde \pi(\cdot|x)~||~\pi(\cdot|x)\Big) \right],\nonumber
\end{align}
where  $D^{\tilde \pi}$ is the group inverse of matrix $I -  P^{\tilde \pi}$, $\epsilon^{\tilde \pi}:=\max\limits_{x\in \X}\Big|\underset{a\sim \tilde \pi(\cdot|x)}{\E}[A^\pi(x,a)]\Big|$.
\end{corollary}

 Lemma \ref{lem:opt_bound}  demonstrates that we use   the best (smallest) norm-wise bound on the difference of stationary distributions in the proof of Theorem \ref{thm:main}.
Lemma \ref{lem:opt_bound} is based on  \cite{Kirkland2008}.

\begin{lemma}\label{lem:opt_bound}
We consider two irreducible and aperiodic transition matrices $P$ and $\tilde P$ with stationary distributions $d$ and $\tilde d$, respectively. We say that $\tau[\tilde P]$ is a condition number of matrix $\tilde P$ if inequality
\begin{align}\label{eq:cond}
\|d - \tilde d\|_1 \leq \tau[\tilde P]\| (P  - \tilde P)^T d \|_1,
\end{align}
holds for any transition matrix $P$.
We let $\tilde D$ be a group inverse matrix of $I-\tilde P$.

 Then  $\tau_1[\tilde D]$ is the smallest condition number:  $ \tau_1[\tilde D]\leq \tau[\tilde P]$
holds for any condition number $\tau(\tilde P)$ satisfying (\ref{eq:cond}).

\end{lemma}

Lemma \ref{lem:opt_bound} shows that inequality (\ref{eq:inv}) in the proof of Theorem~\ref{thm:main}  is a key to the improvement of the policy improvement bounds in \cite{Schulman2015, Achiam2017}. Moreover, it follows from Lemma \ref{lem:opt_bound}  that Corollary \ref{thm:main_av} provides a better policy improvement bound for the average reward criterion than \cite{Dai2020, Zhang2021}.

\section{Interpretation of $\tau_1[ D_\gamma^\pi]$ }

 We provide several  bounds on $\tau_1[D_\gamma^{  \pi}] $ to reveal its dependency on the discount factor $\gamma$ and policy $\pi$.
First, we show how the magnitude of $\tau_1[ D_\gamma^\pi]$  is governed by  the subdominant eigenvalues of the Markov chain.    We let $P$ be an irreducible Markov chain and let $D$ be the group inverse matrix of $I-P$. We define the spectrum of  transition matrix $P$ as $\{1, \lambda_2, \lambda_3, ..., \lambda_{|\X|}\}$, where $|\X|$ is a cardinality of the state space $\X$.  Then, the ergodicity coefficient can be bounded as
\begin{align*}
\tau_1[D]\leq \sum\limits_{i=2}^{|\X|} \frac{1}{1-\lambda_i}=\text{trace}(D),
\end{align*}
see \cite{Seneta1993}.
Matrix $P_\gamma$ defined by (\ref{eq:mod_trans}) is called the Google matrix, and if  the spectrum of  transition matrix $P$ is $\{1, \lambda_2, \lambda_3, ..., \lambda_{|\X|}\}$, then   the spectrum of  the Google  matrix $P_\gamma$ is $\{1, \gamma\lambda_2, \gamma\lambda_3, ..., \gamma\lambda_{|\X|}\}$, see \cite{Haveliwala2003, Langville2003}.  
 Hence, the discounting decreases  the subdominant eigenvalue of the transition matrix that leads to the following bound.
  \begin{lemma}\label{lem:lambda}
  We let $D_\gamma^\pi  $  be the group inverse matrix of $I-P_\gamma^{\pi}$. Then for any discount factor $\gamma\in (0,1]$
 \begin{align*}
  \tau_1[ D_\gamma^\pi ]\leq \sum\limits^{|\X|}_{i=2}\frac{1}{1- \gamma \lambda_i}\leq \frac{|\X|-1}{1- \gamma|\lambda_2|},
 \end{align*}
where $\lambda_2$ is an eigenvalue of $P^\pi$ with the second largest absolute value.
\end{lemma}

 In Lemma~\ref{lem:group} below we derive an alternative upper bound on $\tau_1[D_\gamma^{ \pi}] $.
  For a given policy $\pi$, we assume the transition matrix $P^\pi$ is aperiodic and irreducible. By Proposition 1.7 in \cite{Levin2017}, there exists an integer $\ell$ such that $(P^\pi)^q(x,y)>0$ for all $x,y\in \X$, and $q\geq \ell.$  Then, there exists a sufficiently small constant $\delta^\pi_\mu>0$, such that
   \begin{align}\label{eq:minor}
 (P^\pi)^\ell(x,y)\geq \delta_\mu^\pi \mu(y), \quad \text{ for each }x,y\in \X,
  \end{align}
where $\mu$ denotes the distribution of the initial state.

  \begin{lemma}\label{lem:group}
  We let $D_\gamma^\pi  $  be the group inverse matrix of $I-P_\gamma^{\pi}$.

 We let $\delta_\mu^\pi$ be a constant  that satisfies (\ref{eq:minor}) for $P^\pi$ and some integer $\ell$. Then
  \begin{align*}
  \tau_1[ D_\gamma^\pi ]\leq  \frac{2\ell}{1 - \gamma +  \gamma^\ell \delta^\pi_\mu},
  \end{align*}
  where $\delta_\mu^\pi$ and $\ell$ are  independent of $\gamma$.
\end{lemma}

%
%

\bibliography{KakadeLangford}
\bibliographystyle{icml2021}

\end{document}